\documentclass[twoside]{article}

\usepackage{standalone}

\usepackage[accepted]{my_arxiv}
\usepackage[round]{natbib}

\usepackage{amsmath} 
\usepackage{amsthm}
\usepackage{amsfonts}
\usepackage{esvect}
\usepackage{siunitx}
\usepackage{hyperref}  
\usepackage{algorithm}
\usepackage{arydshln}
\usepackage{booktabs}
\usepackage{wrapfig}
\usepackage[noend]{algpseudocode}
\algrenewcommand{\algorithmiccomment}[1]{\hskip3em// \textit{#1}}

\usepackage{cleveref}
\usepackage{mathtools}
\Crefformat{figure}{#2Fig.~#1#3}
\Crefmultiformat{figure}{Figs.~#2#1#3}{ and~#2#1#3}{, #2#1#3}{ and~#2#1#3}
\Crefformat{section}{#2Sec.~#1#3}
\Crefformat{algorithm}{#2Alg.~#1#3}
\Crefformat{theorem}{#2Thm.~#1#3}
\Crefformat{definition}{#2Def.~#1#3}
\usepackage[acronym,smallcaps]{glossaries}
\DeclareMathOperator*{\argmin}{arg\,min}
\newcommand{\norm}[1]{\left\lVert#1\right\rVert}
\newtheorem{definition}{Definition}[section]
\newtheorem{theorem}{Theorem}[section]

\newtheorem{lemma}{Lemma}[section]
\newtheorem*{remark}{Remark}

\glsdisablehyper
\newacronym{DR}{dr}{\textit{distribution regression}}
\newacronym{KES}{kes}{}
\newacronym{SES}{ses}{}
\usepackage{multirow}
\usepackage{makecell}
\usepackage{bm}
\usepackage[T1]{fontenc}
\usepackage{upquote}
\usepackage{stmaryrd}

\begin{document}

%

%
\runningauthor{Lemercier, Salvi, Damoulas, Bonilla, Lyons}

\twocolumn[

\aistatstitle{Distribution Regression for Sequential Data}

  \aistatsauthor{ Maud Lemercier$^1$, Cristopher Salvi$^2$, Theodoros Damoulas$^1$, Edwin V.~Bonilla$^3$, and Terry Lyons$^2$}

  \aistatsaddress{\\ $^1$University of Warwick \& Alan Turing Institute  \\  \{\tt maud.lemercier, t.damoulas\}@warwick.ac.uk \\ $^2$University of Oxford \& Alan Turing  Institute \\  \{\tt salvi, lyons\}@maths.ox.ac.uk \\ $^3$CSIRO's Data61, \tt edwin.bonilla@data61.csiro.au} 
]

\begin{abstract}
Distribution regression refers to the supervised learning problem where labels are only available for groups of inputs instead of individual inputs. In this paper, we develop a rigorous mathematical framework for distribution regression where inputs are complex data streams.  Leveraging properties of the \textit{expected signature} and a recent \textit{signature kernel trick} for sequential data from stochastic analysis, we introduce two new learning techniques, one feature-based and the other kernel-based. Each is suited to a different data regime in terms of the number of data streams and the dimensionality of the individual streams. We provide theoretical results on the universality of both approaches and demonstrate empirically their robustness to irregularly sampled multivariate time-series, achieving state-of-the-art performance on both synthetic and real-world examples from thermodynamics, mathematical finance and agricultural science.
\end{abstract}

\section{INTRODUCTION}
\Gls{DR} on sequential data describes the task of learning a function from a group of data streams to a single scalar target. For instance, in thermodynamics (\Cref{fig:ideal_gas}) one may be interested in determining the temperature of a gas from the set of trajectories described by its particles \citep{hill1986introduction,reichl1999modern,schrodinger1989statistical}. Similarly in quantitative finance practitioners may wish to estimate mean-reversion parameters from observed market dynamics \citep{papavasiliou2011parameter, gatheral2018volatility, balvers2000mean}. Another example arises in agricultural science where the challenge consists in predicting the overall end-of-year crop yield from high-resolution climatic data across a field \citep{panda2010application, dahikar2014agricultural, you2017deep}. 

\gls{DR} techniques have been successfully applied to handle situations where the inputs in each group are vectors in $\mathbb{R}^d$. Recently, there has been an increased interest in extending these techniques to non-standard inputs such as images \citep{law2017bayesian} or persistence diagrams \citep{kusano2016persistence}. However \gls{DR} for sequential data, such as multivariate time-series, has been largely ignored. The main challenges in this direction are the non-exchangeability of the points in a sequence, which naturally come with an order, and the fact that in many real world secenarios the points in a sequence are irregularly distributed across time. 

\begin{figure*}[ht]
\centering
\includegraphics[width=0.9\textwidth,trim={0 4.1cm 0 0},clip]{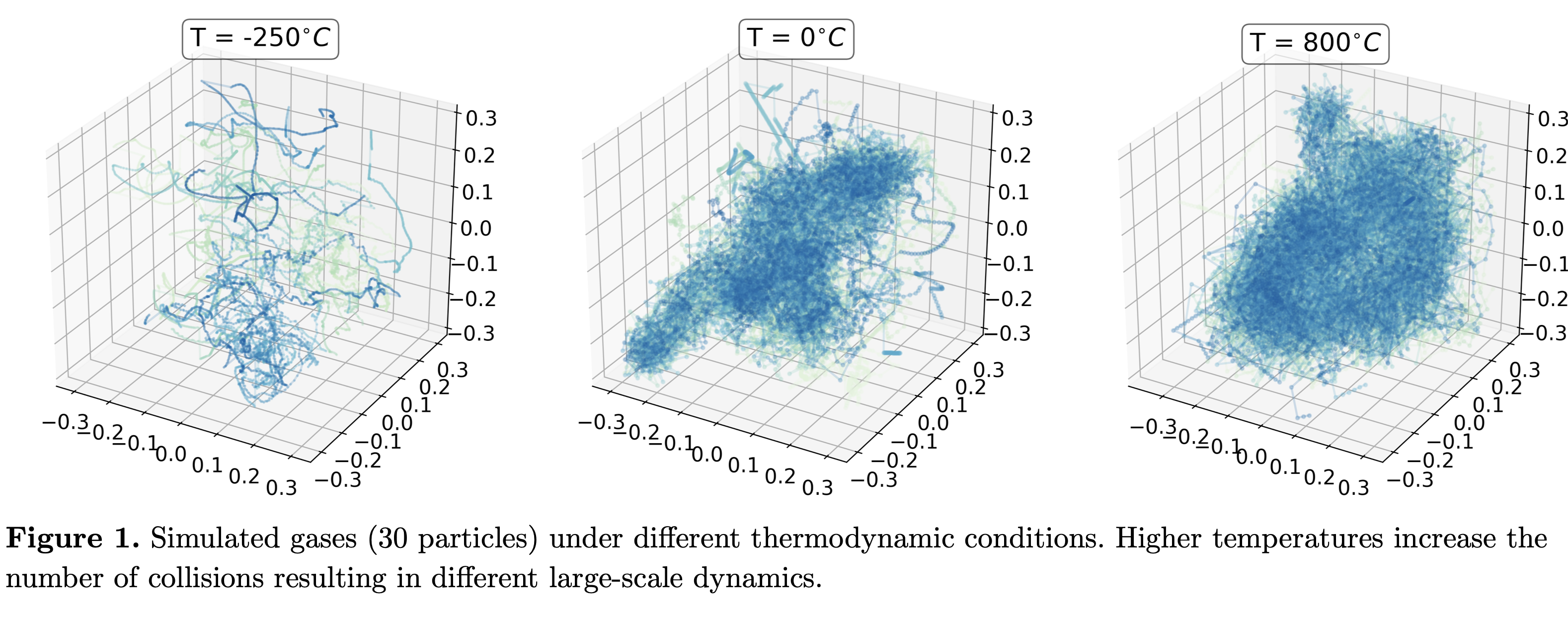}

\caption{Simulation of the trajectories traced by $20$ particles of an ideal gas in a $3$-d box under different thermodynamic conditions. Higher temperatures equate to a higher internal energy in the system which increases the number of collisions resulting in different large-scale dynamics of the gas.}
\label{fig:ideal_gas}
\end{figure*}

In this paper we propose a framework for \acrshort{DR} that addresses precisely the setting where the inputs within each group are complex data streams, mathematically thought of as \textit{Lipschitz continuous paths} (\Cref{sec:background}). We formulate two distinct approaches, one feature-based and the other kernel-based, both relying on a recent tool from stochastic analysis known as the \textit{expected signature} \citep{chevyrev2018signature, chevyrev2016characteristic, lyons2015expected,ni2012expected}. 
Firstly, we construct a new set of features that are universal in the sense that any continuous function on distributions on paths can be uniformly well-approximated by a linear combination of these features (\Cref{sec:pwsig}). Secondly, we introduce a universal kernel on distributions on paths given by the composition of the expected signature and a Gaussian kernel (\Cref{sec:ksig}), which can be evaluated with a kernel trick. The former method is more suitable to datasets containing a large number of low dimensional streams, whilst the latter is better for datasets with a low number of high dimensional streams. We demonstrate the versatility of our methods to handle interacting trajectories like the ones in \Cref{fig:ideal_gas}. We show how these two methods can 
used to provide practical \acrshort{DR} algorithms for time-series, which are robust to irregular sampling and achieve state-of-the-art performance on synthetic and real-world \gls{DR} examples (\Cref{sec:experiments}).

\subsection{Problem definition}\label{sec:pb_def}

Consider $M$ input-output pairs $\{(\{\mathbf{x}^{i,p}\}_{p=1}^{N_i}, y^i)\}_{i=1}^{M}$, where each pair is given by a scalar target $y^i\in\mathbb{R}$ and a group of $N_i$ $d$-dimensional time-series of the form
\begin{align}
    \mathbf{x}^{i,p}=\{(t_1,\mathbf{x}^{i,p}_{1}), \ldots, (t_{\ell_{i,p}},\mathbf{x}^{i,p}_{{\ell_{i,p}}})\},
\end{align}

of possibly unequal lengths $\ell_{i,p}\in\mathbb{N}$, with time-stamps $t_1$ {\small$\leq$} $\ldots$ {\small$\leq$}  $t_{\ell_{i,p}}$ and values $\mathbf{x}^{i,p}_k\in\mathbb{R}^d$. Every d-dimensional time-series $\mathbf{x}^{i,p}$ can be naturally embedded into a Lipschitz-continuous path 
\begin{align}
    x^{i,p}:[t_1,t_{l_{i,p}}] \to \mathbb{R}^d,
\end{align} by piecewise linear interpolation with knots at $t_1,\ldots, t_{l_{i,p}}$ such that $x^{i,p}_{t_k} =  \mathbf{x}^{i,p}_k$. After having formally introduced a set of probability measures on this class of paths, we will summarize the information on each set $\{x^{i,p}\}_{p=1}^{N_i}$ by the \textit{empirical measure} $\delta^i=\frac{1}{N_i}\sum_{p=1}^{N_i}\delta_{x^{i,p}}$ where $\delta_{x^{i,p}}$ is the \textit{Dirac measure} centred at the path $x^{i,p}$. The supervised learning problem we propose to solve consists in learning an unknown function $F:\delta^i\mapsto y^i$. 

\section{THEORETICAL BACKGROUND}\label{sec:background}

We begin by formally introducing the class of paths and the set of probability measures we are considering. 

\subsection{Paths and probability measures on paths}

Let $0$ {\small$\leq$} $a$ {\small$<$} $T$ and $I=[a,T]$ be a closed time interval. Let $E$ be a Banach space of dimension $d \in \mathbb{N}$ (possibly infinite) with norm $\norm{\cdot}_E$. For applications we will take $E:=\mathbb{R}^d$. We denote by $\mathcal{C}(I, E)$ the Banach space \citep{friz2010multidimensional} of Lipschitz-continuous functions $x:I \to E$ equipped with the norm
\begin{equation}
    \norm{x}_{Lip} = \norm{x_a} + \sup_{s,t \in I}\frac{\norm{x_t - x_s}}{|t-s|}.
\end{equation}

We will refer to any element $x \in \mathcal{C}(I,E)$ as an $E$-valued \textit{path}.\footnote{For technical reasons, we remove from $\mathcal{C}(I,E)$ a subset of pathological paths called \textit{tree-like} \citep[Sec.~2.3][]{fermanian2019embedding,hambly2010uniqueness}. This removal has no theoretical or practical impact on what follows.} Given a compact subset of paths $\mathcal{X} \subset \mathcal{C}(I,E)$, we denote by $\mathcal{P}(\mathcal{X})$ the set of (Borel) \textit{probability measures} on $\mathcal{X}$.

The signature has been shown to be an ideal feature map for paths \citep{lyons2014rough}. Analogously, the expected signature is an appropriate feature map for probability measures on paths. Both feature maps take values in the same feature space. In the next section we introduce the necessary mathematical background to describe the structure of this space.

\subsection{A canonical Hilbert space of tensors $\mathcal{T}(E)$}

In what follows $\oplus$ and $\otimes$ will denote the direct sum and the tensor product of vector spaces respectively. For example, $(\mathbb{R}^d)^{\otimes 2} = \mathbb{R}^d \otimes \mathbb{R}^d$ is the space of $d \times d$ matrices and $(\mathbb{R}^d)^{\otimes 3}$ is the space of $d \times d \times d$ tensors. By convention $E^{\otimes 0}=\mathbb{R}$. The following vector space will play a central role in this paper
\begin{equation}
    \mathcal{T}(E) = \bigoplus_{k=0}^\infty E^{\otimes k} = \mathbb{R} \oplus E \oplus E^{\otimes 2} \oplus \ldots 
\end{equation}

If $\{e_1, \ldots , e_d\}$ is a basis of $E$, the elements
$\{e_{i_1} \otimes \ldots \otimes e_{i_k}\}_{(i_1, \ldots , i_k) \in \{1 , \ldots, d\}^k}$ form a basis of $E^{\otimes k}$. For any $A \in \mathcal{T}(E)$ we denote by $A_k \in E^{\otimes k}$ the $k$-tensor component of $A$ and by $A^{(i_1,\ldots,i_k)} \in \mathbb{R}$ its $(i_1\ldots i_k)^{th}$ coefficient. If $E$ is a Hilbert space with inner product $\langle \cdot, \cdot \rangle_E$, then there exists a canonical inner product $\langle \cdot, \cdot \rangle_{E^{\otimes k}}$ on each $E^{\otimes k}$ which extends by linearity to an inner product
\begin{equation}
    \langle A, B \rangle_{\mathcal{T}(E)} = \sum_{k \geq 0} \langle A_k, B_k \rangle_{E^{\otimes k}}
\end{equation}

on $\mathcal{T}(E)$ that thus becomes also a Hilbert space \cite[Sec. 3]{chevyrev2018signature}.

\subsection{The Signature of a path}
\label{sec:path-signature}
The \textit{signature} \citep{chen1957integration, lyons1998differential, lyons2014rough} turns the complex structure of a path $x$ into a simpler vectorial representation given by an infinite sequence of iterated integrals. In this paper, the iterated integrals are defined in the classical \textit{Riemann-Stieltjes} sense.
\begin{definition}
The \textit{signature} $S : \mathcal{C}(I,E) \to \mathcal{T}(E)$ is the map defined elementwise in the following way: the $0^{th}$ coefficient is always $S(x)^{(0)}=1$, whilst all the others are defined as
\begin{equation}\label{eqn:sig}
S(x)^{(i_1\ldots i_k)} =  \underset{a<u_1<\ldots<u_k<T}{\int \ldots \int} d x^{(i_1)}_{u_1} \ldots dx^{(i_k)}_{u_k}\in\mathbb{R}, 
\end{equation}
where $t \mapsto x^{(i)}_t$ denotes the $i^{th}$ path-coordinate of $x$. 
\end{definition}

It is well known that any continuous function on a compact subset of $\mathbb{R}^d$ can be uniformly well approximated by polynomials \cite[Thm.~8.1]{conway2019course}. In full analogy, the collection of iterated integrals defined by the signature provides a basis for continuous functions on compact sets of paths as stated in the following result \cite[Thm. 2.3.5]{fermanian2019embedding}.

\begin{theorem}\label{thm:SW-paths}
Let $\mathcal{X} \subset \mathcal{C}(I, E)$ be a compact set of paths and consider a continuous function $f:\mathcal{X} \to \mathbb{R}$. Then for any $\epsilon > 0$ there exists a truncation level $n\geq0$ such that for any path $x \in \mathcal{X}$
\begin{equation}
    \Big|f(x) - \sum_{k=0}^n\sum_{J \in \{1,\ldots,d\}^k} \alpha_{J}S(x)^{J} \Big|<\epsilon, 
\end{equation}

where $\alpha_J \in \mathbb{R}$ are scalar coefficients.
\end{theorem}

\subsection{Truncating the Signature}\label{ssec:truncation}
In view of numerical applications \citep{bonnier2019deep, graham2013sparse, arribas2018signature, moore2019using, kalsi2020optimal}, the signature of a path $S(x)$ might need to be truncated at a certain level $n \in \mathbb{N}$ yielding the approximation $S^{\leq n}(x) = (1, S(x)_1, \ldots, S(x)_n) \in \mathcal{T}^{\leq n}(E):=\mathbb{R} \oplus E^{\otimes 1} \oplus \ldots E^{\otimes n}$ given by the collection of the first $(d^{n+1}-1)/(d-1)$ iterated integrals in equation (\ref{eqn:sig}). Nonetheless, the resulting approximation is reasonable thanks to \citet[Proposition 2.2]{lyons2007differential} which states that the absolute value of all neglected terms decays factorially as $|S(x)^{(i_1,\ldots,i_n)}|=\mathcal{O}(\frac{1}{n!})$. This factorial decay ensures that when the signature of a path $x$ is truncated, only a negligible amount of information about $x$ is lost \cite[Sec.~1.3]{bonnier2019deep}.

\subsection{Robustness to irregular sampling}\label{ssec:irregular_sampling}
The invariance of the signature to a special class of transformations on the time-domain of a path \cite[Proposition 7.10]{friz2010multidimensional} called time reparametrizations, such as shifting $t \mapsto t+b$ and acceleration $t \mapsto t^b$ $(b\geq 0)$, partially explains its effectiveness to deal with irregularly sampled data-streams \citep{bonnier2019deep, chevyrev2016primer}. In effect, the iterated integrals in equation (\ref{eqn:sig}) disregard the time parametrization of a path $x$, but focus on describing its shape. To retain the information carried by time it suffices to augment the state space of $x$ by adding time $t$ as an extra dimension yielding $t \mapsto \hat{x}_t = (t, x^{(1)}_t,\ldots,x^{(d)}_t)$. This augmentation becomes particularly useful in the case of univariate time-series where the action of the signature becomes somewhat trivial as there are no interesting dependencies to capture between the different path-coordinates \cite[Example 5]{chevyrev2016primer}.

\section{METHODS}\label{sec:methods}

\begin{figure*}[ht]
    \centering
    \includegraphics[scale=0.8]{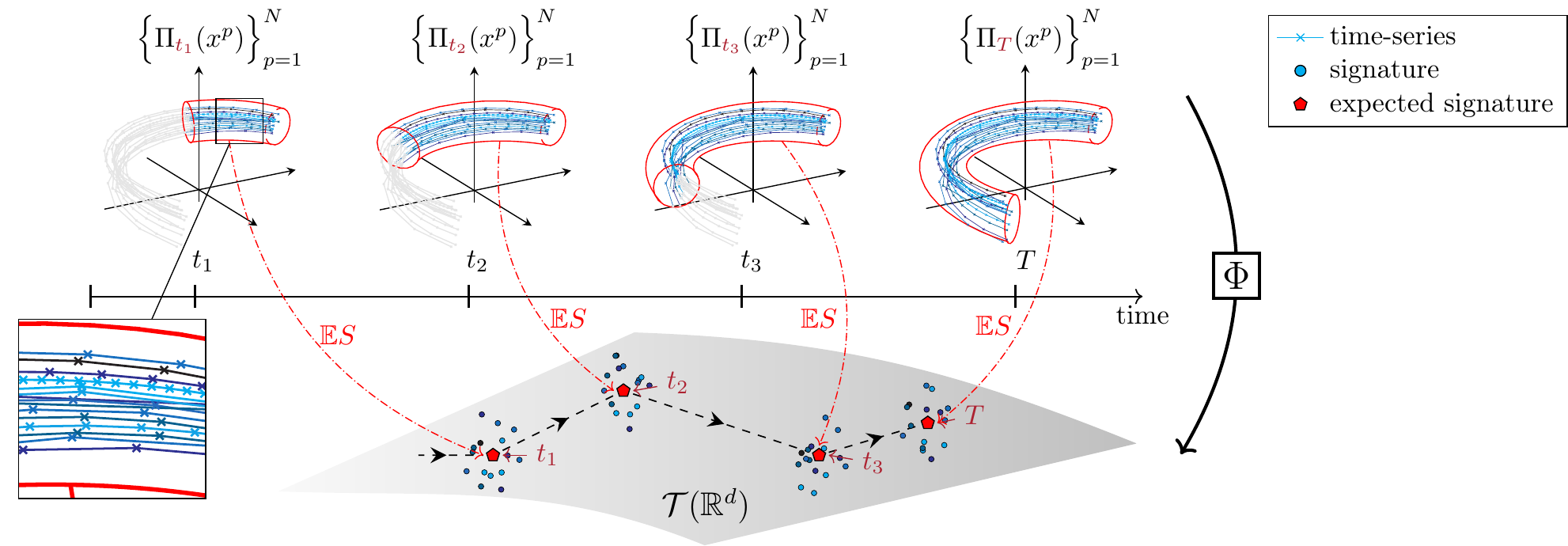}
\caption{Schematic overview of the action of \textit{pathwise expected signature} $\Phi$ on a group of time-series $\{x^p\}_{p=1}^{N}$. (top) Representation of the information about the group of time-series available from start up to time $t_k$. (bottom) At each time $t_k$ this information gets embedded into a single point in $\mathcal{T}(\mathbb{R}^d)$.}
\label{fig:schematic}
\end{figure*}


The \acrfull{DR} setting for sequential data we have set up so far consists of $M$ groups of input-output pairs of the form 
\begin{equation}\label{eqn:data_paths}
    \left\{\left(\{x^{i,p}\in \mathcal{C}(I,E)\}_{p=1}^{N_i}, y^i \in \mathbb{R}\right)\right\}_{i=1}^M,
\end{equation}

such that the finite set of paths $\mathcal{X} = \bigcup_{i=1}^M\{x^{i,p}\}_{p=1}^{N_i}$ is a compact subset of $\mathcal{C}(I,E)$. As mentioned in \Cref{sec:pb_def}, we can summarize the information carried by the collection of paths $\{x^{i,p}\}_{p=1}^{N_i}$ in group $i$ by considering the empirical measure $\delta^i = \frac{1}{N_i}\sum_{p=1}^{N_i}\delta_{x^{i,p}} \in \mathcal{P}(\mathcal{X})$, where $\delta_{x^{i,p}}$ is the Dirac measure centred at the path $x^{i,p}$. In this way the input-output pairs in (\ref{eqn:data_paths}) can be represented as follows
\begin{equation}\label{eqn:data}
    \left\{\left(\delta^i \in \mathcal{P}(\mathcal{X}), y^i \in \mathbb{R}\right)\right\}_{i=1}^M.
\end{equation}

The sequence of moments $(\mathbb{E}[Z^{\otimes m}])_{m\geq 0}$ is classically known to characterize the law $\mu_Z = \mathbb{P} \circ Z^{-1}$ of any finite-dimensional random variable $Z$ (provided the sequence does not grow too fast). It turns out that in the infinite dimensional case of laws of paths-valued random variables (or equivalently of probability measures on paths) an analogous result holds \citep{chevyrev2018signature}. It says that one can fully characterise a probability measure on paths (provided it has compact support) by replacing monomials of a
vector by iterated integrals of a path (i.e.~signatures). At the core of this result is a recent tool from stochastic analysis that we introduce next. 

\begin{definition} 
The \textit{expected signature} is the map $\mathbb{E}S : \mathcal{P}(\mathcal{X}) \to \mathcal{T}(E)$ defined element-wise for any $k\geq 0$ and any $(i_1, \ldots, i_k) \in \{1, \ldots, d\}^k$ as
\begin{equation}
    \mathbb{E}S(\mu)^{(i_1, \ldots, i_k)} = \int_{x \in \mathcal{X}}S(x)^{(i_1, \ldots, i_k)}\mu(dx).
\end{equation}
\end{definition}

We will rely on the following important theorem (see Appendix for a proof) in order to prove the universality of the proposed techniques for \gls{DR} on sequential data presented in the next two sections.

\begin{theorem}\label{thm:continuity-expected-sig}
The \textit{expected signature} map is injective and weakly continuous.
\end{theorem}

\subsection{A feature-based approach (\acrshort{SES})}\label{sec:pwsig} 

As stated in \Cref{thm:SW-paths}, linear combinations of path-iterated-integrals are universal approximators for continuous functions $f$ on compact sets of paths. In this section we prove the analogous density result for continuous functions $F$ on probability measures on paths. We do so by reformulating the problem of \acrshort{DR} on paths as a linear regression on the iterated integrals of an object that we will refer to as the \emph{pathwise expected signature}. We start with the definition of this term followed by the density result. Ultimately, we show that our \acrshort{DR} algorithm materializes as extracting signatures on signatures. 

For any $t \in I=[a,T]$ consider the projection
\begin{equation}
    \Pi_{t} : \mathcal{C}(I, E)\to \mathcal{C}([a,t], E)
\end{equation}

that maps any path $x$ to its restriction to the sub-interval $[a,t] \subset I$, such that $\Pi_t(x) = x_{|_{[a,t]}}$ (see \Cref{fig:schematic}).
\begin{definition}
The \textit{pathwise expected signature} is the function $\Phi: \mathcal{P}(\mathcal{X}) \to \mathcal{C}(I, \mathcal{T}(E))$ that to a probability measure $\mu \in \mathcal{P}(\mathcal{X})$ associates the path $\Phi(\mu):I \to \mathcal{T}(E)$ defined as
\begin{equation}
    \Phi(\mu) : t \mapsto \mathbb{E}_{x \sim \mu}\left[S\left(\Pi_t(x)\right)\right].
\end{equation}
\end{definition}

The action of the function $\Phi$ is illustrated on \Cref{fig:schematic}, and its implementation is outlined in \Cref{algo:Pathwise-E-Sig}.\footnote{Equivalently $\Phi(\mu) = \mathbb{E}S(\Pi_t \# \mu)$ where $\Pi_t \# \mu$ is the push-forward measure of $\mu$ by the measurable map $\Pi_t$.}

\begin{algorithm}[H]
    \caption{Pathwise Expected Signature (PES)}
    \label{algo:Pathwise-E-Sig}
    \begin{algorithmic}[1]
        \State\textbf{Input: } $N$ streams $\{\mathbf{x}^p\}_{p=1}^N$ each of length $\ell$
        \State Create array $\Phi$ to store the PES
        \State Create array $S$ to store the signatures
        \State Initialize  $S[p]\leftarrow 1$ for $p\in\{1,\ldots,N\}$
        \For{each time-step $k\in\{2,\ldots,\ell\}$}
        \State $S[p]\leftarrow S[p]\otimes\exp(\mathbf{x}^p_k-\mathbf{x}^p_{k-1})$ for $p\in\{1,\ldots,N\}$
        \State$\Phi[k]\leftarrow \mathrm{avg}(S)$
        \EndFor
        \State \textbf{Output:}  The pathwise expected signature $\Phi$
    \end{algorithmic}%
\end{algorithm}
In line $6$ of the algorithm we use an algebraic property for fast computation of the signature, known as Chen's relation (detailed in Appendix).
          
         

The next theorem states that any weakly continuous function on $\mathcal{P}(\mathcal{X})$ can be uniformly well approximated by a linear combination of terms in the signature of the pathwise expected signature.


\begin{theorem}\label{thm:Terry}
Let $\mathcal{X} \subset \mathcal{C}(I,E)$ be a compact set of paths and consider a weakly continuous function $F:\mathcal{P}(\mathcal{X})\to\mathbb{R}$. Then for any $\epsilon > 0$ there exists a truncation level $m\geq 0$ such that for any probability measure $\mu \in \mathcal{P}(\mathcal{X})$
\begin{equation}
    \Big|F(\mu) - \sum_{k=0}^m\sum_{J \in \{1,\ldots,d\}^k} \alpha_{J}S(\Phi(\mu))^{J} \Big|<\epsilon, 
\end{equation}

where $\alpha_J \in \mathbb{R}$ are scalar coefficients.
\end{theorem}

\begin{proof}
$\mathcal{P}(\mathcal{X})$ is compact (see proof of \Cref{thm:kernel}) and the image of a compact set by a continuous function is compact. Therefore, the image $K=\Phi(\mathcal{P}(\mathcal{X}))$ is a compact subset of $\mathcal{C}(I,\mathcal{T}(E))$. Consider a weakly continuous function $F:\mathcal{P}(\mathcal{X}) \to \mathbb{R}$. Given that $\Phi$ is injective (see Appendix), $\Phi$ is a bijection when restricted to its image $K$. Hence, there exists a continuous function $f:K \to \mathbb{R}$ (w.r.t $\norm{\cdot}_{Lip}$) such that $F = f\circ \Phi$. By \Cref{thm:SW-paths} we know that for any $\epsilon > 0$, there exists a linear functional $\mathcal{L} : \mathcal{T}(E) \to \mathbb{R}$ such that $\norm{f - \mathcal{L} \circ S}_\infty < \epsilon$. Thus $\|F \circ \Phi^{-1} - \mathcal{L} \circ S\|_\infty < \epsilon$, implying $\|F - \mathcal{L} \circ S\circ\Phi\|_\infty<\epsilon$.
\end{proof}

The practical consequence of this theorem is that the complex task of learning a highly non-linear regression function $F : P(\mathcal{X}) \to \mathbb{R}$ can be reformulated as a linear regression on the signature (truncated at level $m$) of the pathwise expected signature (truncated at level $n$). The resulting \acrshort{SES} algorithm is outlined in Alg.\ref{algo:SES}, and has time complexity $\mathcal{O}(M\ell d^n(N + d^m))$, where $M$ is the total number of groups, $\ell$ is the largest length across all time-series, $d$ is the state space dimension, $N$ is the maximum number of input time series in a single group and $n,m$ are the truncation levels of the signature and pathwise expected signature respectively.
\begin{algorithm}[h]
    \caption{\acrshort{DR} on sequential data with \acrshort{SES}}
    \label{algo:SES}
    \begin{algorithmic}[1]
         \State\textbf{Input: }  $\{(\{\mathbf{x}^{i,p}\}_{p=1}^{N_i}, y^i)\}_{i=1}^{M}$
         \State Create array $A$ to store $M$ signatures of the PES.
            \For{ each group $i\in\{1,...,M\}$} 
          
             \State $\Phi = \mathrm{PES}(\{\mathbf{x}^{i,p}\}_{p=1}^{N_i})$ \Comment{Using \Cref{algo:Pathwise-E-Sig}}
             \For{each time-step $k\in\{2,\ldots,\ell_i\}$} 
             \\\Comment{Compute the signature of the PES}\\
             \Comment{via Chen's relation}
             \State$A[:,i] \leftarrow A[:,i]\otimes\exp{(\Phi_k-\Phi_{k-1})}$
             \EndFor
         
            \EndFor
            \State $(\alpha_0, \ldots , \alpha_c)$ $\leftarrow$ LinearRegression($A,(y^i)_{i=1}^M$) 
            \State\textbf{Output: }Regression coefficients $(\alpha_0, \ldots , \alpha_c)$ 
    \end{algorithmic}%
\end{algorithm}
\subsection{A kernel-based approach (\acrshort{KES}) }\label{sec:ksig}

The \acrshort{SES} algorithm is well suited to datasets containing a possibly large number $MN$ of relatively low dimensional paths. If instead the input paths are high dimensional, it would be prohibitive to deploy \acrshort{SES} 
since the number of terms in the signature increases exponentially in the dimension $d$ of the path. To address this, in this section we construct a new kernel function $k:\mathcal{P}(\mathcal{X}) \times \mathcal{P}(\mathcal{X}) \to \mathbb{R}$ combining the expected signature with a Gaussian kernel and prove its universality to approximate weakly continuous function on probability measures on paths. The resulting kernel-based algorithm (\acrshort{KES}) for \gls{DR} on sequential data is well-adapted to the opposite data regime to the one above, i.e.~when the dataset consists of few number $MN$ of high dimensional paths.





\begin{theorem}\label{thm:kernel} Let $\mathcal{X} \subset \mathcal{C}(I,E)$ be a compact set of paths and $\sigma>0$. The kernel $k:\mathcal{P}(\mathcal{X})\times\mathcal{P}(\mathcal{X})\to\mathbb{R}$ defined by 
\begin{equation}\label{eqn:univ_kernel}
k(\mu,\nu)=\exp{\Big(-\sigma^2\norm{\mathbb{E}S(\mu)-\mathbb{E}S(\nu)}_{\mathcal{T}(E)}^2\Bigr)},
\end{equation}
is universal, i.e. the associated RKHS is dense in the space of continuous functions from $\mathcal{P}(\mathcal{X})$ to $\mathbb{R}$. 
\end{theorem}
\begin{proof}
By \citet[Thm. 2.2]{christmann2010universal} if $K$ is a compact metric space and $H$ is a separable Hilbert space such that there exists a continuous and injective map $\rho : K \to H$, then for $\sigma >0$ the Gaussian-type kernel $k_\sigma : K \times K \to \mathbb{R}$ is a universal kernel, where $k_\sigma(z,z') = \exp{\Big(-\sigma^2\norm{\rho(z)-\rho(z')}_{H}^2\Bigr)}$. With the metric induced by $\norm{\cdot}_{Lip}, \mathcal{X}$ is a compact metric space. Hence the set $\mathcal{P}(\mathcal{X})$ is weakly-compact \cite[Thm. 10.2]{walkden2014ergodic}. By \Cref{thm:continuity-expected-sig}, the \textit{expected signature} $\mathbb{E}S:\mathcal{P}(\mathcal{X})\to\mathcal{T}(E)$ is injective and weakly continuous. Furthermore $\mathcal{T}(E)$ is a Hilbert space with a countable basis, hence it is separable. Setting $K=\mathcal{P}(\mathcal{X}), H = \mathcal{T}(E)$ and $\rho = \mathbb{E}S$ concludes the proof. 
\end{proof} 

\subsection{Evaluating the universal kernel $k$}

When the input measures are two empirical measures $\delta^1=\frac{1}{N_1}\sum_{p=1}^{N_1}\delta_{x^{1,p}}$ and $\delta^2=\frac{1}{N_2}\sum_{q=1}^{N_2}\delta_{x^{2,q}}$, the evaluation of the kernel $k$ in \Cref{eqn:univ_kernel} requires the ability to compute the tensor norm on $\mathcal{T}(E)$
 \begin{align}\label{eqn:mmd}
 &\norm{\mathbb{E}S(\delta^1)-\mathbb{E}S(\delta^2)}^2 = E_{11}+E_{22}-2E_{12},\nonumber\\
 &E_{ij} =\frac{1}{N_iN_j} \sum_{p,q=1}^{N_i,N_j}\hspace{-2pt}\bigl\langle S(x^{i,p}),S(x^{j,q})\bigr\rangle, ~ i,j\in\{1,2\}
 \end{align}
where all the inner products are in $\mathcal{T}(E)$. Each of these inner products defines another recent object from stochastic analysis called the signature kernel $k_{sig}$ \citep{kiraly2019kernels}. Recently,  \citet{cass2020computing} have shown that $k_{sig}$ is actually the solution of a surprisingly simple partial differential equation (PDE). This result provides us with a ``kernel trick'' for computing the inner products in \Cref{eqn:mmd} by a simple call to any numerical PDE solver of choice.  

\begin{theorem}\citep[Thm. 2.2]{cass2020computing}\label{thm:pde_sig_kernel}
The signature kernel defined as 
\begin{equation}
    k_{sig}(x,y) := \langle S(x),S(y)\rangle_{\mathcal{T}(E)}
\end{equation}

is the solution $u:[a,T]\times[a,T] \to \mathbb{R}$ at $(s,t)=(T,T)$ of the following linear hyperbolic PDE
\begin{align}
    \frac{\partial^2 u}{\partial s \partial t} = (\dot{x}_s^T\dot{y}_t)u && u(a,\cdot)=1,~u(\cdot,a)=1.
\end{align}
\end{theorem}
 
In light of \Cref{thm:kernel}, \gls{DR} on paths with \acrshort{KES} can be performed via any kernel method \citep{drucker1997support,quinonero2005unifying} available within popular libraries \citep{scikit-learn,de2017gpflow,gardner2018gpytorch} using the Gram matrix computed via \Cref{algo:KES} and leveraging the aforementioned kernel trick. When using a finite difference scheme (referred to as PDESolve in \Cref{algo:KES}) to approximate the solution of the PDE, the resulting time complexity of \acrshort{KES} is $\mathcal{O}(M^3+M^2N^2\ell^2d)$.


\begin{algorithm}[h]
    \caption{Gram matrix for \acrshort{KES}}
    \label{algo:KES}
    \begin{algorithmic}[1]
         \State\textbf{Input: }  $\{x^{i,p}\}_{p=1}^{N_i},~i=1,\ldots,M$ and $\sigma>0$.
         \State Initialize $0$-array $G \in \mathbb{R}^{M\times M}$  
            \For{each pair of groups $(i,j)$ such that $i\leq j$} 
            \State Initialize $0$-array $K_{ij}\in \mathbb{R}^{N_i\times N_j}$
            \State Similarly initialize $K_{ii}, K_{jj}\in \mathbb{R}^{N_i\times N_i},\mathbb{R}^{N_j\times N_j}$
            \For{$p,p'$ in group $i$ and $q,q'$ in group $j$}
            \State $ K_{ii}[p,p']\leftarrow \mathrm{PDESolve}(x^{i,p},x^{i,p'})$

            \State $ K_{jj}[q,q']\leftarrow \mathrm{PDESolve}(x^{j,q},x^{j,q'})$

            \State $ K_{ij}[p,q]\leftarrow \mathrm{PDESolve}(x^{i,p},x^{j,q})$
            \EndFor
            \State $G[i,j]\leftarrow \mathrm{avg}(K_{ii})+\mathrm{avg}(K_{jj})-2\times\mathrm{avg}(K_{ij})$ 
            \State $G[j,i]\leftarrow G[i,j]$
            \EndFor
            \State $G\leftarrow\exp(-\sigma^2 G)$ \Comment{element-wise $exp$}
            \State\textbf{Output:} The gram matrix $G$. 
    \end{algorithmic}%
\end{algorithm}

\paragraph{Remark}
In the case where the observed paths are assumed to be i.i.d. samples $\{x^p\}_{p=1}^N \sim \mu$ from the law of an underlying random process one would expect the bigger the sample size $N$, the better the approximation of $\mu$, and therefore of its expected signature $\mathbb{E}S(\mu)$. Indeed, for an arbitrary multi-index $\tau = (i_1,\ldots,i_k)$, the \textit{Central Limit Theorem} yields the convergence (in distribution) 
\begin{align*}
    \sqrt{N}\Bigl(\mathbb{E}_{x\sim\mu}[S^{\tau}(x)]-\frac{1}{N}\sum_{p=1}^{N}S^{\tau}(x^p)\Bigr)\overset{\mathcal{D}}{\to}\mathcal{N}(0,\sigma_\tau^2)
\end{align*} 

as the variance $\sigma_\tau^2
= \mathbb{E}_{x \sim \mu}[S^\tau(x)^2] - (\mathbb{E}_{x \sim \mu}[S^\tau(x)])^2$ is always finite; in effect for any path $x$, the product $S^{\tau}(x)S^{\tau}(x)$ can always be expressed as a finite sum of higher-order terms of $S(x)$ \citep[Thm. 1]{chevyrev2016primer}.

\section{RELATED WORK}\label{sec:related_work}
Recently, there has been an increased interest in extending regression algorithms to the case where inputs are sets of numerical arrays \citep{hamelijnck2019multi,law2018variational,musicant2007supervised,wagstaff2008multiple,skianis2019rep}. Here we highlight the previous work most closely related to our approach.
\paragraph{Deep learning techniques} DeepSets \citep{zaheer2017deep} are examples of neural networks designed to process each item of a set individually, aggregate the outputs by means of well-designed operations (similar to pooling functions) and feed the aggregated output to a second neural network to carry out the regression. However, these models depend on a large number of parameters and results may largely vary with the choice of architecture and activation functions \citep{wagstaff2019limitations}. 

\paragraph{Kernel-based techniques} In the setting of \acrshort{DR}, elements of a set are viewed as samples from an underlying probability distribution \citep{szabo2016learning, law2017bayesian, muandet2012learning, flaxman2015machine, smola2007hilbert}. This framework can be intuitively summarized as a two-step procedure. Firstly, a probability measure $\mu$ is mapped to a point in an RKHS $\mathcal{H}_1$ by means of a \textit{kernel mean embedding} $\Phi: \mu \to \int_{x \in \mathcal{X}}k_1(\cdot,x)\mu(dx)$, where $k_1: \mathcal{X} \times \mathcal{X} \to  \mathbb{R}$ is the associated reproducing kernel. Secondly, the regression is finalized by approximating a function $F:\mathcal{H}_1\to\mathbb{R}$ via a minimization of the form $F \approx \argmin_{g \in \mathcal{H}_2}\sum_{i=1}^M\mathcal{L}(y^i, g\circ\Phi(\mu^i))$, where $\mathcal{L}$ is a loss function, resulting in a procedure involving a second kernel $k_2:\mathcal{H}_1 \times \mathcal{H}_1 \to \mathbb{R}$. Despite the theoretical guarantees of these methods \citep{szabo2016learning}, the feature map $k_1(\cdot, x)$ acting on the support $\mathcal{X}$ is rarely provided explicitly, especially in the setting of non-standard input spaces $\mathcal{X} \not\subset \mathbb{R}^d$, requiring manual adaptations to make the data compatible with standard kernels. 
In \Cref{sec:experiments} we denote by DR-$k_1$ the models produced by choosing $k_2$ to be a Gaussian-type kernel.

\paragraph{The signature method}
The signature method consists in using the terms of the signature as features to solve supervised learning problems on time-series, with successful applications for detection of bipolar disorder \citep{arribas2018signature} and human action recognition \citep{yang2017leveraging} to name a few. The signature features have been used to construct neural network layers \citep{bonnier2019deep,graham2013sparse} in deep architectures. To the best of our knowledge, we are the first to use signatures in the context of \acrshort{DR}.

\section{EXPERIMENTS}\label{sec:experiments}

We benchmark our feature-based (\acrshort{SES}) and kernel-based (\acrshort{KES}) methods  against DeepSets and the existing kernel-based \gls{DR} techniques  discussed in \Cref{sec:related_work} on various simulated and real-world examples from physics, mathematical finance and agricultural science. With these examples, we show the ability of our methods to handle challenging situations where only a few number of labelled groups of multivariate time-series are available. We consider the kernel-based techniques DR-$k_1$ with $k_1\in\{\mathrm{RBF},\mathrm{Matern32},\mathrm{GA}\}$, where GA refers to the Global Alignment kernel for time-series from \citet{cuturi2007kernel}. For \acrshort{KES} and DR-$k_1$ we perform Kernel Ridge Regression, whilst for \acrshort{SES} we use Lasso Regression. All models are run $5$ times and we report the mean and standard deviation of the predictive mean squared error (MSE). Other metrics are reported in the Appendix. The hyperparameters of \acrshort{KES}, \acrshort{SES} and DR-$k_1$ are selected by cross-validation via a grid search on the training set of each run. Additional details about hyperparameters search and model architecture, as well as the code to reproduce the experiments can be found in the supplementary material.

   
   
\subsection{A defective electronic device}\label{ssec:circuits}
We start with a toy example to show the robustness of our methods to irregularly sampled time-series.
\begin{figure}[h]
 \begin{center}
    \includegraphics[width=0.99\columnwidth]{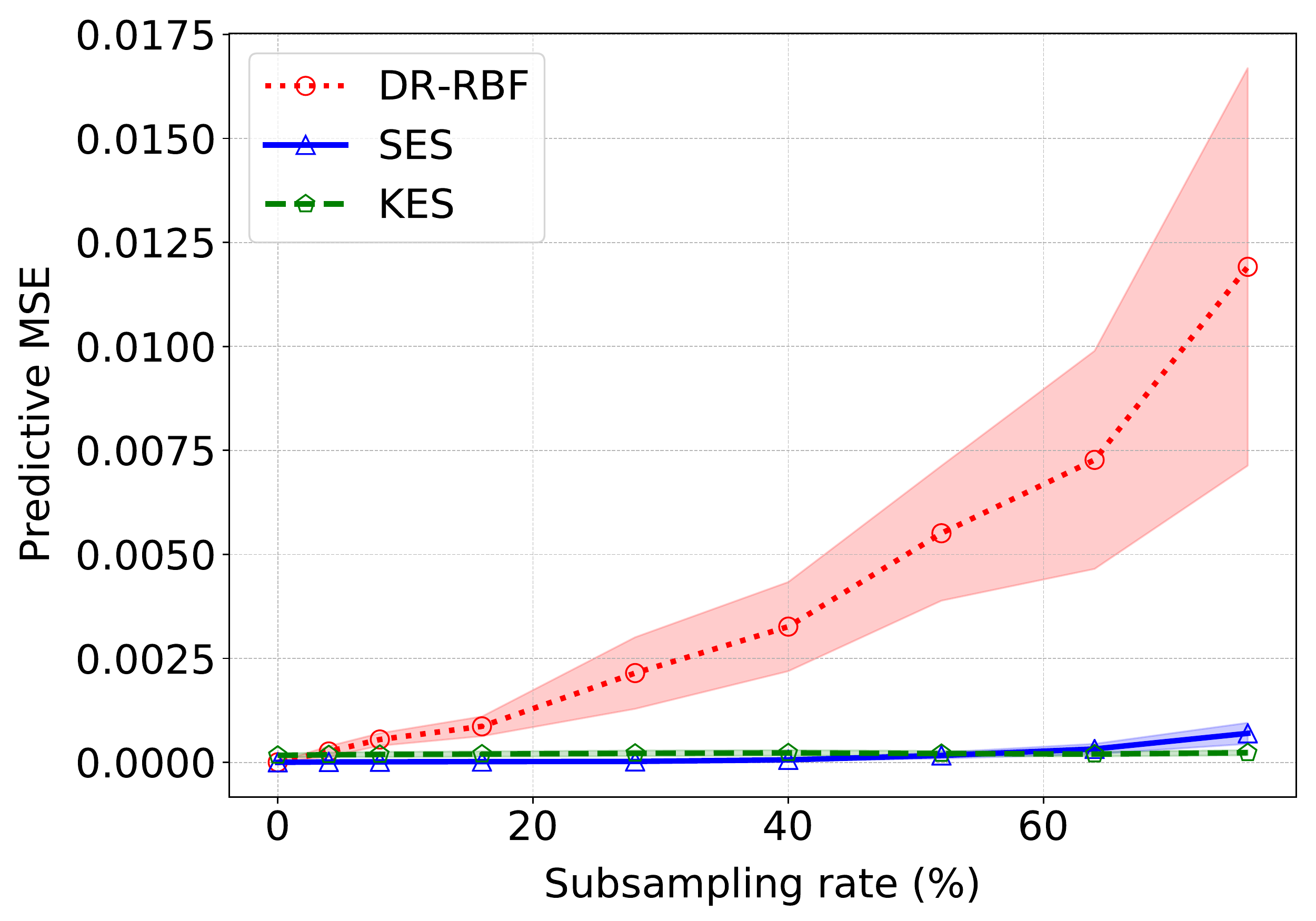}
  \end{center}
  \caption{Predictive MSE at various subsampling rates for $M=50$ circuits and $N=15$ devices. The shaded area indicates the standard deviation.}
  \label{fig:defective_device}
\end{figure}%
For this, we propose to infer the phase $\varphi$ of an electronic circuit from multiple recordings of its voltage $v^{\varphi}(t)=\sin(\omega t)$ and current $i^{\varphi}(t)=\sin(\omega t-\varphi)$. The data consists of $M$ simulated circuits with phases $\{\varphi^i\}_{i=1}^M$ selected uniformly at random from $[\pi/8, \pi/2]$. Each circuit is attached to $N$ measuring devices recording the two sine waves over $20$ periods at a frequency $25$ points per period. We then randomly subsample the data at rates ranging from $0\%$ to $75\%$ independently for each defective device. As shown in \Cref{fig:defective_device}, the predictive performances of DR-RBF drastically deteriorate when the subsampling rate increases, whilst results for \acrshort{KES} and \acrshort{SES} remain roughly unchanged. 
\subsection{Inferring the temperature of an ideal gas}\label{ssec:particles}


\begin{table}[h]
\caption{Ideal gas dataset. Radii of all particles composing the simulated gases: $r_1=3.5\cdot10^{-1}(V/N)^3$ (few collisions) and  $r_2=6.5\cdot10^{-1}(V/N)^3$ (many collisions).}
\begin{center}
\begin{tabular}{l c c} 
 \toprule
   \multirowcell{2}[0pt][l]{Model} &  \multicolumn{2}{c}{Predictive MSE [$\times 10^{-2}$]}\\ 
    \cmidrule(lr){2-3} 
    & $r_1$  & $r_2>r_1$  \\ \midrule
    DeepSets & 8.69~(3.74) & 5.61~(0.91) \\
    DR-RBF & 3.08~(0.39) &  4.36~(0.64)\\
     DR-Matern32 & 3.54~(0.48) &  4.12~(0.39)\\
    DR-GA & 2.85~(0.43) &  3.69~(0.36)\\
    \cdashline{1-3}\noalign{\vskip 0.5ex}
      \acrshort{KES} &\textbf{1.31}~(0.34) &  \textbf{0.08}~(0.02) \\
    \acrshort{SES} & \textbf{1.26}~(0.23)&  \textbf{0.09}~(0.03) \\ \bottomrule
\end{tabular}%
\label{table:particles}
\end{center}
\end{table}

The thermodynamic properties of an \textit{ideal gas} of $N$ particles inside a $3$-d box of volume $V$ ($3$ cm\raise0.5ex\hbox{3}) can be described in terms of the temperature $T$ (K), the pressure $P$ (Pa) and the total energy $U$ (J) via the two equations of state $PV=Nk_BT$ and $U=c_VNk_BT$, where $k_B$ is the \textit{Boltzmann constant} \citep{adkins1983equilibrium}, and $c_V$ the heat capacity. The large-scale behaviour of the gas can be related to the trajectories of the individual particles (through their \textit{momentum = mass $\times$ velocity}) by the equation $U=\frac{1}{2}\sum_{p=1}^Nm_p |\vv{v_p}|^2$. The complexity of the large-scale dynamics of the gas depends on $T$ (see \Cref{fig:ideal_gas}) as well as on the radius of the particles. For a fixed $T$, the larger the radius the higher the chance of collision between the particles. We simulate $M=50$ different gases of $N=20$ particles each by randomly initialising all velocities and letting particles evolve at constant speed.\footnote{We assume \citep{idealGas} that the environment is frictionless, and that particles are not subject to other forces such as gravity. We make use of python code from \url{https://github.com/labay11/ideal-gas-simulation}.} The task is to learn $T$ (sampled uniformly at random from $[1,1\,000]$) from the set of trajectories traced by the particles in the gas. In Table \ref{table:particles} we report the results of two experiments, one where particles have a small radius (few collisions) and another where they have a bigger radius (many collisions). The performance of DR-$k_1$ is comparable to the ones of \acrshort{KES} and \acrshort{SES} in the simpler setting. However, in the presence of a high number of collisions our models become more informative to retrieve the global temperature from local trajectories, whilst the performance of DR-$k_1$ drops with the increase in system-complexity. With a total number of $MN=1\,000$ time-series of dimension $d=7$ (after path augmentation discussed in \Cref{ssec:irregular_sampling} and in the Appendix), \acrshort{KES} runs in $50$ seconds, three times faster than \acrshort{SES} on a $128$ cores CPU.

\subsection{Parameter estimation in a pricing model}\label{ssec:finance}

 \begin{table}[h]
 \caption{Predictive MSE (standard deviation) on the rough volatility dataset. $N$ is the number of rough volatility trajectories and $(M,d,\ell)=(50,2,200)$.}
 \begin{center}
 \resizebox{1\columnwidth}{!}{%
 \begin{tabular}{l ccc}
  \toprule
   \multirowcell{2}[0pt][l]{Model} &  \multicolumn{3}{c}{Predictive MSE [$\times 10^{-3}$]}\\ 
     \cmidrule(lr){2-4} 
     & N=20 & N=50 & N=100 \\ \midrule
        DeepSets &74.43~(47.57) & 74.07~(49.15) & 74.03~(47.12) \\
     DR-RBF & 52.25~(11.20) & 58.71~(19.05) & 44.30~(7.12) \\
   
     DR-Matern32 & 48.62~(10.30) &
     54.91~(12.02) & 32.99~(5.08)\\
   
       DR-GA & 3.17~(1.59) & 2.45~(2.73) & 0.70~(0.42)\\
     \cdashline{1-4}\noalign{\vskip 0.5ex}
          \acrshort{KES}& \textbf{1.41}~(0.40) & \textbf{0.30}~(0.07) & \textbf{0.16}~(0.03)  \\
    \acrshort{SES}& \textbf{1.49}~(0.39)  & \textbf{0.33}~(0.12)& \textbf{0.21}~(0.05)\\ \bottomrule
 \end{tabular}%
 }
 \label{table:rough-vol}
 \end{center}
 \end{table}
 
Financial practitioners often model asset prices via an SDE of the form $dP_t = \mu_tdt + \sigma_tdW_t$, where $\mu_t$ is a drift term, $W_t$ is a $1$-d Brownian motion (BM) and $\sigma_t$ is the volatility process \citep{arribas2020sigsdes}. This setting is often too simple to match the volatility observed in the market, especially since the advent of electronic trading \citep{gatheral2018volatility}. Instead, we model the (rough) volatility process as $\sigma_t = \exp\{P_t\}$ where $dP_t = -a(P_t-m)dt + \nu dW_t^H$ is a \textit{fractional Ornstein-Uhlenbeck} (fOU) process, with $a, \nu, m$ {\footnotesize $\geq$} $0$. The fOU is driven by a \textit{fractional Brownian Motion} (fBM) $W_t^H$ of \textit{Hurst exponent} $H \in (0,1)$, governing the regularity of the trajectories \citep{decreusefond1999stochastic}.\footnote{We note that sample-paths of fBM are not in $\mathcal{C}([0,T], \mathbb{R})$ but we can assume that the interpolations obtained from market high-frequency data provide a sufficiently refined approximation of the underlying process.} In line with the findings in \citet{gatheral2018volatility} we choose $H=0.2$ and tackle the task of estimating the mean-reversion parameter $a$ from simulated sample-paths of $\sigma_t$. We consider $50$ mean-reversion values $\{a^i\}_{i=1}^{50}$ chosen uniformly at random from $[10^{-6}, 1]$. Each $a^i$ is regressed on a collection of $N=20,50,100$ (time-augmented) trajectories $\{\hat{\sigma}_t^{i,p}\}_{p=1}^{N}$ of length $200$. As shown in Table \ref{table:rough-vol}, \acrshort{KES} and \acrshort{SES} systematically yield the best MSE among all compared models. Moreover, the performance of \acrshort{KES} and \acrshort{SES} progressively improves with the number of time series in each group, in accordance to the remark at the end of \Cref{sec:methods}, whilst this pattern is not observed for DR-RBF, DR-Matern32, and DeepSets. Both \acrshort{KES} and \acrshort{SES} yield comparable performances. However, whilst the running time of \acrshort{SES} remains stable ($\approx 1~\mathrm{min}$) when $MN$ increases from $1\,000$ to $5\,000$, the running time of \acrshort{KES} increases from $\approx 1~\mathrm{min}$ to $15~\mathrm{min}$ (on $128$ cores).


\subsection{Crop yield prediction from GLDAS data}\label{ssec:crops}

Finally, we evaluate our methods on a crop yield prediction task. The challenge consists in predicting the yield of wheat crops over a region from the longitudinal measurements of climatic variables recorded across different locations of the region. We use the publicly available Eurostat dataset containing the total annual regional yield of wheat crops in mainland France---divided in $22$ administrative regions---from $2015$ to $2017$. The climatic measurements (temperature, soil humidity and precipitation) are extracted from the GLDAS database \citep{doi:10.1175/BAMS-85-3-381}, are recorded every $6$ hours at a spatial resolution of $0.25\si{\degree}\times 0.25\si{\degree}$, and their number varies across regions.\footnote{http://ec.europa.eu/eurostat/data/database} We further subsample at random $50\%$ of the measurements. \acrshort{SES} and \acrshort{KES} are the two methods which improve the most against the baseline which consists in predicting the average yield on the train set (Table \ref{table:crops}). 

 \begin{table}[h]
\caption{MSE and MAPE (mean absolute percentage error) on the Eurostat/GLDAS dataset}
\begin{center}
\begin{tabular}{l cc} 
    \toprule
    Model  &  MSE &  MAPE \\ 
   \midrule
    Baseline & 2.38~(0.60) & 23.31~(4.42) \\ 
    \cdashline{1-3}\noalign{\vskip 0.5ex}
    DeepSets & 2.67~(1.02) & 22.88~(4.99)  \\
    DR-RBF & 0.82~(0.22)& 13.18~(2.52) \\
    DR-Matern32 & 0.82~(0.23) & 13.18~(2.53)\\
    DR-GA & 0.72~(0.19) &12.55~(1.74)   \\
    \cdashline{1-3}\noalign{\vskip 0.5ex}
    \acrshort{KES} & 0.65~(0.18) & 12.34~(2.32)   \\
   \acrshort{SES}& \textbf{0.62}~(0.10) & \textbf{10.98}~(1.12)  \\
   \bottomrule
\label{table:crops}
\end{tabular}
\end{center}
\end{table}%

\section{CONCLUSION}
We have developed two novel techniques for \emph{distribution regression on sequential data}, a task largely ignored in the previous literature. In the first technique, we introduce the pathwise expected signature and construct a universal feature map for probability measures on paths. In the second technique, we  define a universal kernel based on the expected signature. We have shown the robustness of our proposed methodologies to irregularly sampled multivariate time-series, achieving state-of-the-art performances on various \acrshort{DR} problems for sequential data. Future work will focus on developing algorithms to handle simultaneously a large number of groups of high-dimensional time-series. 

\subsubsection*{Acknowledgments}
We deeply thank Dr Thomas Cass and Dr Ilya Chevyrev for the very helpful discussions. ML was supported by the OxWaSP CDT under the EPSRC grant EP/L016710/1. CS was supported by the EPSRC grant EP/R513295/1. ML and TD were supported by the Alan Turing Institute under the EPSRC grant EP/N510129/1. CS and TL were supported by the Alan Turing Institute under the EPSRC grant EP/N510129/1.

\bibliographystyle{apalike}
\bibliography{AISTATS_draft}

\appendix 
\onecolumn
\section{PROOFS}

In this section, we prove that the expected signature $\mathbb{E}S$ is weakly continuous (\Cref{sec:weak-continuity-ES}), and that the pathwise expected signature $\Phi$ is injective and weakly continuous (\Cref{sec:continuity-PES}).

Recall that in the main paper we consider a compact subset of paths $\mathcal{X} \subset \mathcal{C}(I,E)$, where $I$ is a closed interval and $E$ is a Banach space of dimension $d$ (possibly infinite, but countable). We will denote by $\mathcal{P}(\mathcal{X})$ the set of Borel probability measures on $\mathcal{X}$ and by $S(\mathcal{X}) \subset \mathcal{T}(E)$ the image of $\mathcal{X}$ by the signature $S: \mathcal{C}(I, E) \to \mathcal{T}(E)$.

As shown in \cite[Section 3]{chevyrev2018signature}, if $E$ is a Hilbert space with inner product $\langle \cdot, \cdot \rangle_E$, then for any $k\geq 1$ the following bilinear form defines an inner product on $E^{\otimes k}$
\begin{equation}
    \big\langle e_{i_1} \otimes \ldots \otimes e_{i_k}, e_{j_1} \otimes \ldots \otimes e_{j_k} \big\rangle_{E^{\otimes k}} = \prod_{r=1}^k \delta_{i_r,j_r}, \hspace{0.5cm} \delta_{ij} =
    \begin{cases}
            1, &         \text{if } i=j,\\
            0, &         \text{if } i\neq j.
    \end{cases}
\end{equation}
which extends by linearity to an inner product $\langle A, B \rangle_{\mathcal{T}(E)} = \sum_{k \geq 0} \langle A_k, B_k \rangle_{E^{\otimes k}}$ on $\mathcal{T}(E)$ that thus becomes also a Hilbert space.

\subsection{Weak continuity of the expected signature}\label{sec:weak-continuity-ES}

\begin{definition}\label{def:weak}
A sequence of probability measures $\mu_n \in \mathcal{P}(\mathcal{X})$ \textit{converges weakly} to $\mu$ if for every $f \in C_b(\mathcal{X}, \mathbb{R})$ we have $\int_\mathcal{X}fd\mu_n \to \int_\mathcal{X}fd\mu$ as $n \to \infty$, where $C_b(\mathcal{X}, \mathbb{R})$ is the space of real-valued continuous bounded functions on $\mathcal{X}$.
\end{definition}

\begin{remark}
Since $\mathcal{X}$ is a compact metric space, we can drop the word "bounded" in \Cref{def:weak}.
\end{remark}

\begin{definition}
Given two probability measures $\mu, \nu \in \mathcal{P}(\mathcal{X})$, the Wasserstein-1 distance is defined as follows 
\begin{equation}
    W_1(\mu, \nu) = \inf_{\gamma \in \Pi(\mu, \nu)} \int_{x,y \in \mathcal{X}}\norm{x-y}_{Lip} d\gamma(x,y)
\end{equation}
where the infimum is taken over all possible couplings of $\mu$ and $\nu$.
\end{definition}

\begin{lemma}\cite[Theorem 5.3]{chevyrev2018signature}
The signature $S:\mathcal{C}(I,E) \to \mathcal{T}(E)$ is injective.\footnote{Up to tree-like equivalence (see \cite[appendix B]{chevyrev2018signature} for a definition and detailed discussion).}
\end{lemma}

\begin{lemma}\cite[Corollary 5.5]{chevyrev2016characteristic}\label{thm:sig-cont}
The signature $S:\mathcal{C}(I,E) \to \mathcal{T}(E)$ is continuous w.r.t. $\norm{\cdot}_{Lip}$.
\end{lemma}

\begin{lemma}\cite[Theorem 5.6]{chevyrev2018signature}\label{thm:es-injective}
The expected signature $\mathbb{E}S : \mathcal{P}(\mathcal{X}) \to \mathcal{T}(E)$ is injective.\footnote{This result was firstly proved in \cite{fawcett2002problems} for probability measures supported on compact subsets of $\mathcal{C}(I,E)$, which is enough for this paper. It was also proved in a more abstract setting in \cite{chevyrev2016characteristic}. The authors of \cite{chevyrev2018signature} introduce a normalization that is not needed in case of compact supports, as they mention in \cite[(I) - page 2]{chevyrev2018signature}}
\end{lemma}

\begin{theorem}
The expected signature $\mathbb{E}S : \mathcal{P}(\mathcal{X}) \to \mathcal{T}(E)$ is weakly continuous.
\end{theorem}

\begin{proof}
Consider a sequence $\{\mu_n\}_{n \in \mathbb{N}}$ of probability measures on $\mathcal{P}(\mathcal{X})$ converging weakly to a measure $\mu \in \mathcal{P}(\mathcal{X})$. By \Cref{thm:sig-cont} the signature $S: x \mapsto S(x)$ is continuous w.r.t. $\norm{\cdot}_{Lip}$. Hence, by definition of weak-convergence (and because $\mathcal{X}$ is compact), for any $k>0$ and any multi-index $(i_1, \ldots, i_k) \in \{1, \ldots, d\}^k$ it follows that $\int_{x \in \mathcal{X}}S(x)^{(i_1, \ldots, i_k)}\mu_n(dx) \to \int_{x \in \mathcal{X}}S(x)^{(i_1, \ldots, i_k)}\mu(dx)$. The factorial decay given by \cite[Proposition 2.2]{lyons2007differential} yields  $\int_{x \in \mathcal{X}}S(x)\mu_n(dx) \to \int_{x \in \mathcal{X}}S(x)\mu(dx)$ in the topology induced by $\langle \cdot, \cdot \rangle_{\mathcal{T}(E)}$.
\end{proof}

\subsection{Injectivity and weak continuity of the pathwise expected signature}\label{sec:continuity-PES}
\begin{theorem}\cite[Theorem 3.7]{lyons2007differential}\label{thm:extension}
Let $x \in \mathcal{C}(I,E)$ and recall the definition of the projection $\Pi_t: x \mapsto x_{|_{[a,t]}}$. Then, the $\mathcal{T}(E)$-valued path defined by
\begin{equation}
    S_{path}(x): t \mapsto S \circ \Pi_t(x)
\end{equation}
is Lipschitz continuous. Furthermore the map $x \mapsto S_{path}(x)$ is continuous w.r.t. $\norm{\cdot}_{Lip}$.
\end{theorem}

\begin{theorem}
The pathwise expected signature $\Phi: \mathcal{P}(\mathcal{X}) \to \mathcal{C}(I, \mathcal{T}(E))$ is injective.\footnote{For any $\mu \in \mathcal{P}(\mathcal{X})$ the path $\Phi(\mu) \in \mathcal{C}(I,\mathcal{T}(E))$. Indeed $\Phi(\mu)$ is a continuous path because $x$, $\Pi_t$, $S$ and $\Phi$ are all continuous and the composition of continuous functions is continuous. The Lipschitzianity comes from the fact that $\norm{\Phi(\mu)}_{Lip} \leq \mu(\mathcal{X}) \sup_{x \in \mathcal{X}}\norm{S_{path}(x)}_{Lip} < +\infty$ by \Cref{thm:extension}.}
\end{theorem}

\begin{proof}
Let $\mu, \nu \in \mathcal{P}(\mathcal{X})$ be two probability measures. If $\Phi(\mu) = \Phi(\nu)$, then for any $t \in I$, $\mathbb{E}_{x \sim \mu}[S\circ \Pi_t(x)] = \mathbb{E}_{y \sim \nu}[S\circ \Pi_t(y)]$. In particular, for $t=T$, $\mathbb{E}S(\mu) = \mathbb{E}_{x \sim \mu}[S(x)] = \mathbb{E}_{y \sim \nu}[S(y)] = \mathbb{E}S(\nu)$. The result follows from the injectivity of the expected signature $\mathbb{E}S$ (\Cref{thm:es-injective}). 
\end{proof}


\begin{theorem}\label{thm:continuity}
The pathwise expected signature $\Phi:\mathcal{P}(\mathcal{X}) \to \mathcal{C}(I, \mathcal{T}(E))$ is weakly continuous.
\end{theorem}

\begin{proof}
Let $\{\mu_n\}_{n \in \mathbb{N}}$ be a sequence in $\mathcal{P}(\mathcal{X})$ converging weakly to $\mu \in \mathcal{P}(\mathcal{X})$. As $S_{path}$ is continuous (\Cref{thm:extension}), it follows, by the \textit{continuous mapping theorem}, that $S_{path}\#\mu_n \to S_{path}\#\mu$ weakly, where $S_{path}\#\mu$ is the pushforward measure of $\mu$ by $S_{path}$. Given that $S_{path}$ is continuous and $\mathcal{X}$ is compact, it follows that the image $S_{path}(\mathcal{X})$ is a compact subset of the Banach space $\mathcal{C}(I, \mathcal{T}(E))$. 
By \cite[Theorem 6.8]{villani2008optimal} weak convergence of probability measures on compact supports is equivalent to convergence in \textit{Wasserstein-1 distance}. By \textit{Jensen's inequality} $\norm{\mathbb{E}[S_{path}\#\mu_n] - \mathbb{E}[S_{path}\#\mu]}_{Lip} \leq \mathbb{E}[\norm{S_{path}\#\mu_n - S_{path}\#\mu}_{Lip}]$. Taking the infimum over all couplings $\gamma \in \Pi(S_{path}\#\mu_n, S_{path}\#\mu)$ on the right-hand-side of the previous equation we obtain $\norm{\mathbb{E}[S_{path}\#\mu_n] - \mathbb{E}[S_{path}\#\mu]}_{Lip} \leq W_1(S_{path}\#\mu_n, S_{path}\#\mu) \to 0$, which yields the convergence $\mathbb{E}[S_{path}\#\mu_n] \to \mathbb{E}[S_{path}\#\mu]$ in $\norm{\cdot}_{Lip}$ over $\mathcal{C}(I,\mathcal{T}(E))$. Noting that $\mathbb{E}[S_{path}\#\mu] = \Phi(\mu)$ concludes the proof.
\end{proof}

\section{EXPERIMENTAL DETAILS}

In our experiments we benchmark \acrshort{KES} and \acrshort{SES} against DeepSets and DR-$k_1$ where $k_1\in\{\mathrm{RBF},\mathrm{Matern32},\mathrm{GA}\}$. Both \acrshort{KES} and \acrshort{SES} do not take into account the length of the input time-series. Apart from DR-GA, all other baselines are designed to operate on vectorial data. Therefore, in order to deploy them in the setting of \acrshort{DR} on sequential data, manual pre-processing (such as padding) is required. In the next section we describe how we turn discrete time-series into continuous paths on which the signature operates. 
\subsection{Transforming discrete time-series into continuous paths}
Consider a $d$-dimensional time-series of the form $\mathbf{x}=\{(t_1,\mathbf{x}_{1}), \ldots, (t_{\ell},\mathbf{x}_{\ell})\}$ with time-stamps $t_1$ {\small$\leq$} $\ldots$ {\small$\leq$}  $t_{\ell}$ and values $\mathbf{x}_k\in\mathbb{R}^d$, and the continuous path $x$ obtained by linearly interpolating between the points $\mathbf{x}_1,\cdots,\mathbf{x}_{\ell}$. The signature (truncated at level $n$) of $x$ can be computed explicitly with existing Python packages \cite{reizenstein2018iisignature,esig,signatory}, does not depend on the time-stamps $(t_1,\ldots,t_{\ell_{i,j}})$, and produces $(d^{n+1}-1)/(d-1)$ terms when $d>1$. When $d=1$ the signature is trivial since $S^{\leq n}(x)=(1,(x_{t_\ell}-x_{t_1}), \frac{1}{2}(x_{t_\ell}-x_{t_1})^2,\cdots,\frac{1}{n!}(x_{t_\ell}-x_{t_1})^n)$. As mentioned in \Cref{ssec:irregular_sampling} we can simply augment the paths with a monotonous coordinate, such that $\hat{x}:t\mapsto (t,x_t)$, where $t\in[a,T]$, effectively reintroducing a time parametrization. Another way to augment the state space of the data and obtain additional signature terms is the \textit{lead-lag} transformation (see \Cref{def:lead-lag}) which turns a $1$-d data stream into a $2$-d path. For example if the data stream is $\{1,5,3\}$ one obtains the $2$-d continuous path $\hat{x}:t\mapsto(x_t^{(lead)},x_t^{(lag)})$ where $x^{(lead)}$ and $x^{(lag)}$ are the linear interpolations of $\{1,5,5,3,3\}$ and $\{1,1,5,5,3\}$ respectively. A key property of the lead-lag transform is that the difference between $S(\hat{x})^{(1,2)}$ and $S(\hat{x})^{(2,1)}$ is the quadratic variation $QV(x)=\sum_{k=1}^{\ell-1}(x_{t_{k+1}}-x_{t_k})^2$ \cite{chevyrev2016primer}. Hence, even when $d>1$, it may be of interest to lead-lag transform the coordinates of the paths for which the quadratic variation is important for the task at hand.
\begin{definition}[Lead-lag]\label{def:lead-lag} Given a sequence of points $\mathbf{x}=\{\mathbf{x}_{1},\ldots,\mathbf{x}_{\ell}\}$ in $\mathbb{R}^d$ the lead-lag transform yields two new sequences $\mathbf{x}^{(lead)}$ and $\mathbf{x}^{(lag)}$ of length $2\ell-1$ of the following form
\begin{align*}
   \mathbf{x}^{(lead)}_{p}= \left\{
    \begin{array}{ll}
        \mathbf{x}_{k} & \mathrm{if}~ p=2k-1 \\
        \mathbf{x}_{k} & \mathrm{if}~ p=2k-2.
    \end{array}
\right. &&
   \mathbf{x}^{(lag)}_{p}= \left\{
    \begin{array}{ll}
        \mathbf{x}_{k} & \mathrm{if}~ p=2k-1 \\
        \mathbf{x}_{k} & \mathrm{if}~ p=2k.
    \end{array}
\right.
\end{align*}
\end{definition}


In our experiments we add time and lead-lag all coordinates except for the first task which consists in inferring the phase of an electronic circuit (see \Cref{ssec:circuits} in the main paper).

\subsection{Implementation details}\label{ssec:implementation}
The distribution regression methods (including DR-$k_1$, \acrshort{KES} and \acrshort{SES}) are implemented on top of the Scikit-learn library \cite{scikit-learn}, whilst we use the existing codebase \url{https://github.com/manzilzaheer/DeepSets} for DeepSets.

\subsubsection{KES} 
The \acrshort{KES} algorithm relies on the signature kernel trick which is referred to as PDESolve in the main paper. In the algorithm below we outline the finite difference scheme we use for the experiments. In all the experiments presented in the main paper, the discretization level of the PDE solver is fixed to $n=0$ such that the time complexity to approximate the solution of the PDE is $\mathcal{O}(d\ell^2)$ where $\ell$ is the length of the longest data stream.  

\begin{algorithm}[H]
    \caption{PDESolve}
    \label{algo:PDESolve}
    \begin{algorithmic}[1]
        \State\textbf{Input: } two streams $\{\mathbf{x}_k\}_{k=1}^{\ell_\mathbf{x}}$, $\{\mathbf{y}_k\}_{k=1}^{\ell_\mathbf{y}}$ of dimension $d$ and discretization level $n$ (step size $=2^{-n}$)
        \State Create array $U$ to store the solution of the PDE
        \State Initialize  $U[i,:]\leftarrow 1$ for $i\in\{1,2,\ldots,2^n*(\ell_x-1)+1\}$
        \State Initialize  $U[:,j]\leftarrow 1$ for $j\in\{1,2,\ldots,2^n*(\ell_y-1)+1\}$
        \For{each $i\in\{1,2,\ldots,2^n*(\ell_\mathbf{x}-1)\}$}
         \For{each $j\in\{1,2,\ldots,2^n*(\ell_\mathbf{y}-1)\}$}
        \State $\Delta_\mathbf{x} = (\mathbf{x}_{\left\lceil i/(2^n) \right\rceil+1}-\mathbf{x}_{\left\lceil i/(2^n) \right\rceil})/2^n$
        \State $\Delta_\mathbf{y} = (\mathbf{y}_{\left\lceil j/(2^n) \right\rceil+1}-\mathbf{y}_{\left\lceil j/(2^n) \right\rceil})/2^n$
        \State $U[i+1,j+1] = U[i,j+1] + U[i+1,j] + (\Delta_\mathbf{x}^T\Delta_\mathbf{y}-1.)*U[i,j]$
        \EndFor
        \EndFor
        \State \textbf{Output:} The solution of the PDE at the final times $U[-1,-1]$  
    \end{algorithmic}%
\end{algorithm}

\subsubsection{SES} 

The \acrshort{SES} algorithm from the main paper relies on an algebraic property for fast computation of signatures, known as Chen's relation. Given a piecewise linear path $x = \Delta x_{t_2}\star \ldots \star \Delta x_{t_\ell}$ given by the concatenation $\star$ of individual increments $\mathbb{R}^d\ni\Delta x_{t_k}=x_{t_k}-x_{t_{k-1}}, ~k=2,\ldots,\ell$, one has $S(x) = \exp(\Delta x_{t_2}) \otimes \ldots \otimes \exp(\Delta x_{t_\ell})$, where $\exp$ denotes the tensor exponential and $\otimes$ the tensor product. Using Chen's relation, computing the signature (truncated at level $n$) of a sequence of length $\ell$ has complexity $\mathcal{O}(\ell d^n)$.

\subsubsection{Baselines}
\begin{wraptable}{r}{0.5\textwidth}
\vspace{-1\baselineskip}
\caption{Kernels $k_1$ for the kernel-based baselines. See \cite{cuturi2017soft} for the definition of $\mathrm{dtw}_{1/\gamma}$ in the GA kernel.}
    \begin{tabular}{l c}
    \toprule
       RBF  & $\exp(-\gamma^2\norm{x-x'}^2)$ \\
       Matern32 &  $(1+\sqrt{3}\gamma^2\norm{x-x'})\exp(-\sqrt{3}\gamma^2\norm{x-x'})$ \\
       GA & $\exp(-\gamma\,\mathrm{dtw}_{1/\gamma}(x,x'))$\\
       \bottomrule
    \end{tabular}%
    \vspace{-0.3cm}
\end{wraptable}
For the kernel-based baselines DR-$k_1$, we perform Kernel Ridge regression with the kernel defined by $k(\delta^i,\delta^j)=\exp{(-\sigma^2\norm{\rho(\delta^i)-\rho(\delta^j)}_{\mathcal{H}_1}^2)}$, where $\rho(\delta^i)=N_i^{-1}\sum_{p=1}^{N_i}k_{1}(\cdot,x^{i,p})$. For $k_1\in\{\mathrm{RBF},\mathrm{Matern32}\}$, if the time-series are multi-dimensional, the dimensions are stacked to form one large vector $x\in\mathbb{R}^{d\ell}$. See Table 1 for the expressions of the kernels $k_1$ used as baselines.


For DeepSets, the two neural networks are feedforward neural networks with ELU activations. We train by minimizing the mean squared error.

\subsection{Hyperparameter selection}
All models are run $5$ times. The hyperparameters of \acrshort{KES}, \acrshort{SES} and DR-$k_1$ are selected by cross-validation via a grid search on the training set ($80\%$ of the data selected at random) of each run. The range of values for each parameter is specified in \Cref{table:grid_search}. 

\begin{table}[H]
\begin{center}
\caption{Range of values for each parameter of DR-$k_1$, \acrshort{KES} and \acrshort{SES}. We denote by $\alpha$ the regularization parameter in Kernel Ridge regression and Lasso regression. The kernels parameters $\gamma$ and $\sigma$ are expressed in terms of lengthscales $\ell_1$ and $\ell_2$ such that $\gamma^2=1/(2\ell_1^2)$ and $\sigma^2=1/(2\ell_2^2)$.}
\resizebox{0.98\columnwidth}{!}{%
\begin{tabular}{l c c c c c}
 \toprule
  Model &  $\ell_1$ & $\ell_2$ & $\alpha$ & $n$ & $m$ \\ 
  \hline\\[-1.5ex]
    \text{DR-RBF} & $\{10^{-3},10^{-2},\ldots,10^{2},10^3\}$ & $\{10^{-3},10^{-2},\ldots,10^{2},10^3\}$ & $\{10^{-3},10^{-2},\ldots,10^{2},10^3\}$ & N/A & N/A  \\
      \text{DR-Matern32} & $\{10^{-3},10^{-2},\ldots,10^{2},10^3\}$ & $\{10^{-3},10^{-2},\ldots,10^{2},10^3\}$ & $\{10^{-3},10^{-2},\ldots,10^{2},10^3\}$ & N/A & N/A  \\
        \text{DR-GA} & $\{7\cdot10^{1},7\cdot10^{2}\}$ & $\{10^{-3},10^{-2},\cdots,10^{2},10^{3}\}$ & $\{10^{-3},10^{-2},\ldots,10^{2},10^3\}$ & N/A & N/A  \\
    \acrshort{KES} & N/A & $\{10^{-3},10^{-2},\ldots,10^{2},10^3\}$ & $\{10^{-3},10^{-2},\ldots,10^{2},10^3\}$ & N/A & N/A  \\
    \acrshort{SES} & N/A & N/A & $\{10^{-5}, 10^{-4},\ldots, 10^{4}, 10^{5}\}$ & $\{2,3\}$ & $\{2\}$ \\ \bottomrule
\end{tabular}%
}
\label{table:grid_search}
\end{center}
\end{table}

\section{ADDITIONAL RESULTS}\label{sec:additional_results}

\subsection{Additional performance metrics}
We report the mean absolute percentage error (MAPE) as well as the computational time on two synthetic examples (the ideal gas and the rough volatility examples). As discussed in the main paper, these two datasets represent two data regimes: in one case (the rough volatility model) there is a high number of low dimensional time-series (see \Cref{fig:rough_vol}), whist in the other case (ideal gas), there is a relatively small number of time-series with a higher state-space dimension. Apart from DeepSets (which is run on a GPU), all other models are run on a $128$ cores CPU. 
\begin{figure}[H]
\centering
\includegraphics[scale=0.4]{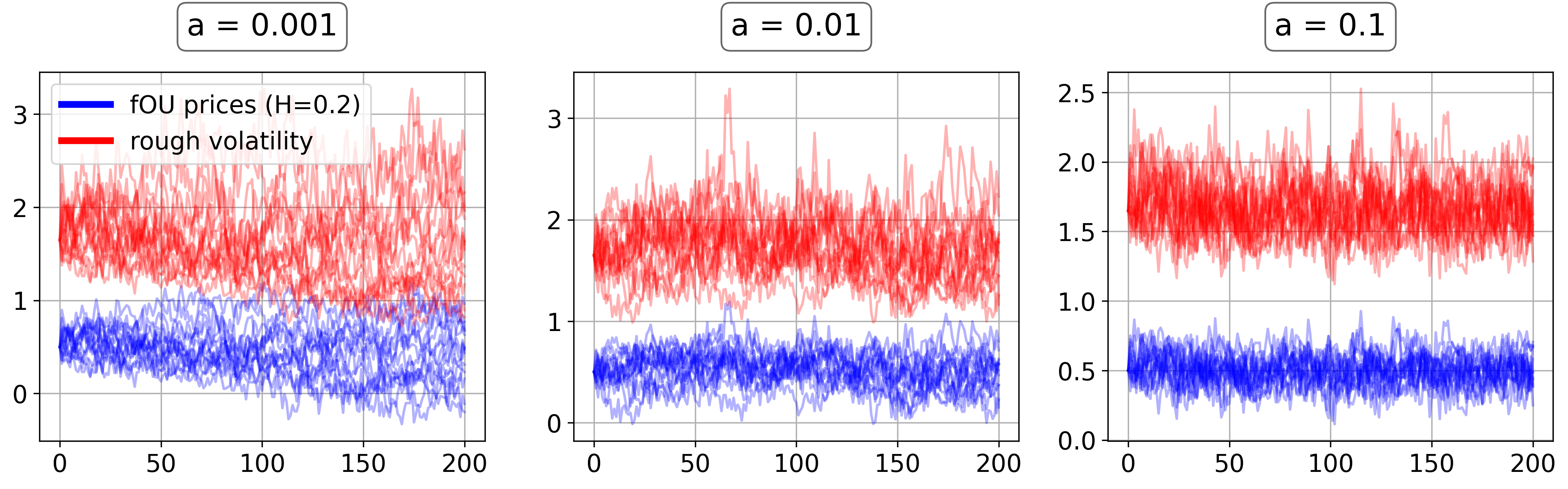}
\caption{fOU sample-paths and corresponding volatility for $3$ different mean-reversion parameters.}
\label{fig:rough_vol}
\end{figure}

\begin{table}[H]
\caption{Ideal gas example.}
\begin{center}
\begin{tabular}{l c c c } 
 \toprule
   \multirowcell{2}[0pt][l]{Model} &   \multicolumn{2}{c}{Predictive MAPE } & \multirowcell{2}[0pt][l]{Time (s)} \\ 
    \cmidrule(lr){2-3} 

     & $r_1$  & $r_2>r_1$ &\\ \midrule
    DeepSets &  82.50~(20.20) & 53.49~(13.93) & 31 \\
    DR-RBF  & 32.09~(5.78) & 41.15~(11.81) & 58\\
     DR-Matern32  & 33.79~(5.16) & 40.20~(12.45) & 55\\
    DR-GA  & 31.61~(5.60) & 39.17~(13.87) & 68\\
    \cdashline{1-4}\noalign{\vskip 0.5ex}
      \acrshort{KES} & 16.57~(4.86) & 4.20~(0.79) & 49\\
    \acrshort{SES}  & 15.75~(2.65) & 4.44~(1.36) & 120 \\ \bottomrule
\end{tabular}%
\label{table:particles2}
\end{center}
\end{table}

 \begin{table}[H]
 \caption{Rough volatility example.}
 \begin{center}
 \begin{tabular}{l ccc ccc}
  \toprule
   \multirowcell{2}[0pt][l]{Model} &  \multicolumn{3}{c}{Predictive MAPE } & \multicolumn{3}{c}{Time (min) }\\ 
     \cmidrule(lr){2-4} 
       \cmidrule(lr){5-7} 
      & N=20 & N=50 & N=100 & N=20 & N=50 & N=100\\ \midrule
        DeepSets  & 44.85~(17.80)& 44.75~(17.93)& 45.00 ~(18.21)& 1.31 & 1.86 &  2.68\\
     DR-RBF & 43.86~(13.36) & 45.54~(10.05)& 41.00~(12.98) & 0.71& 1.38& 7.50\\
   
     DR-Matern32 & 40.97~(10.81) &43.59~(9.79) & 35.35~(9.18) &  0.73 & 1.00 & 7.80 \\
   
       DR-GA  & 11.94~(7.14) &9.54~(6.85) & 5.51~(2.78)& 0.68 & 2.60 & 9.80 \\
     \cdashline{1-7}\noalign{\vskip 0.5ex}
          \acrshort{KES}&  6.12~(1.00)& 2.83~(0.49)&   2.07~(0.42) & 0.71 & 4.00 & 15.50\\
    \acrshort{SES}& 6.67~(3.35) & 3.58~(0.84)& 2.14~(0.62) & 0.60 & 0.65 &  0.78\\ \bottomrule
 \end{tabular}
 \label{table:rough-vol2}
 \end{center}
 \end{table}


\subsection{Interpretability}\label{sec:additional_exp}

\begin{wrapfigure}{r}{0.35\textwidth}
 \begin{center}
    \includegraphics[width=0.35\textwidth]{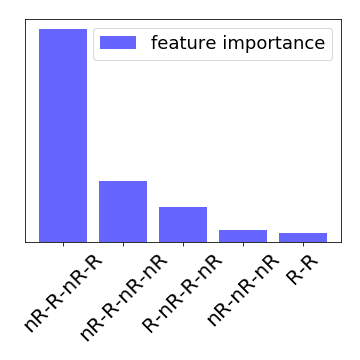}
  \end{center}
  \caption{{\footnotesize The $5$ most predictive features provided by (Lasso) \acrshort{SES} for the task of crop yield prediction.}}
    \label{fig:NDVI}
\end{wrapfigure}

When dealing with complex data-streams, cause-effect relations between the different path-coordinates might be an essential feature that one wishes to extract from the signal. Intrinsic in the definition of the signature is the concept of iterated integral of a path over an ordered set of time indices  $a<u_1<\ldots<u_k<T$. This ordering of the domain of integration, naturally captures causal dependencies between the coordinate-paths $x^{(i_1)}, \ldots , x^{(i_k)}$. 

Taking this property into account, we revisit the crop yield prediction example (see \Cref{ssec:crops} in the main paper, and \Cref{fig:crops}) to show how the iterated integrals from the signature (of the pathwise expected signature) provide interpretable predictive features, in the context of \gls{DR} with \acrshort{SES}. For this, we replace the climatic variables by two distinct multi-spectral reflectance signals: 1) near-infrared (nR) spectral band; 2) red (R) spectral band \cite{huete2002overview}. These two signals are recorded at a much lower temporal resolution than the climatic variables, and are typically used to assess the health-status of a plant or crop, classically summarized by the \textit{normalized difference vegetation index} (NDVI) \cite{huete2002overview}. To carry out this experiment, we use a publicly available dataset \cite{hubert2019time} which contains multi-spectral time-series corresponding to geo-referenced French wheat fields from $2006$ to $2017$, and consider these field-level longitudinal observations to predict regional yields (still obtained from the Eurostat database).\footnote{http://ec.europa.eu/eurostat/data/database}Instead of relying on a predefined vegetation index signal, such as the aforementionned NDVI : $t\mapsto (x_t^{nR}-x_t^{R})/(x_t^{nR}+x_t^{R})$, we use the raw signals in the form of $2$-dimensional paths $x:t\mapsto x_t=(x^{nR}_t, x^{R}_t)$ to perform a Lasso \gls{DR} with \acrshort{SES}.

\paragraph{Interpretation} Chlorophyll strongly absorbs light at wavelengths around $0.67 \mu m$ (red) and reflects strongly in green light, therefore our eyes perceive healthy vegetation as green. Healthy plants have a high reflectance in the near-infrared between $0.7$ and $1.3 \mu m$. This is primarily due to healthy internal structure of plant leaves \cite{rahman2004ndvi}. Therefore, this absorption-reflection cycle can be seen as a good indicator of the health of crops. Intuitively, the healthier the crops, the higher the crop-yield will be at the end of the season. It is clear from \Cref{fig:NDVI} that the feature in the signature that gets selected by the Lasso penalization mechanism corresponds to a double red-infrared cycle, as described above. This simple example shows how the terms of the signature are not only good predictors, but also carry a natural interpretability that can help getting a better understanding of the underlying physical phenomena.  

\begin{figure}[h]
    \centering
    \includegraphics[scale=0.4,trim={0 0.8cm 0 0},clip]{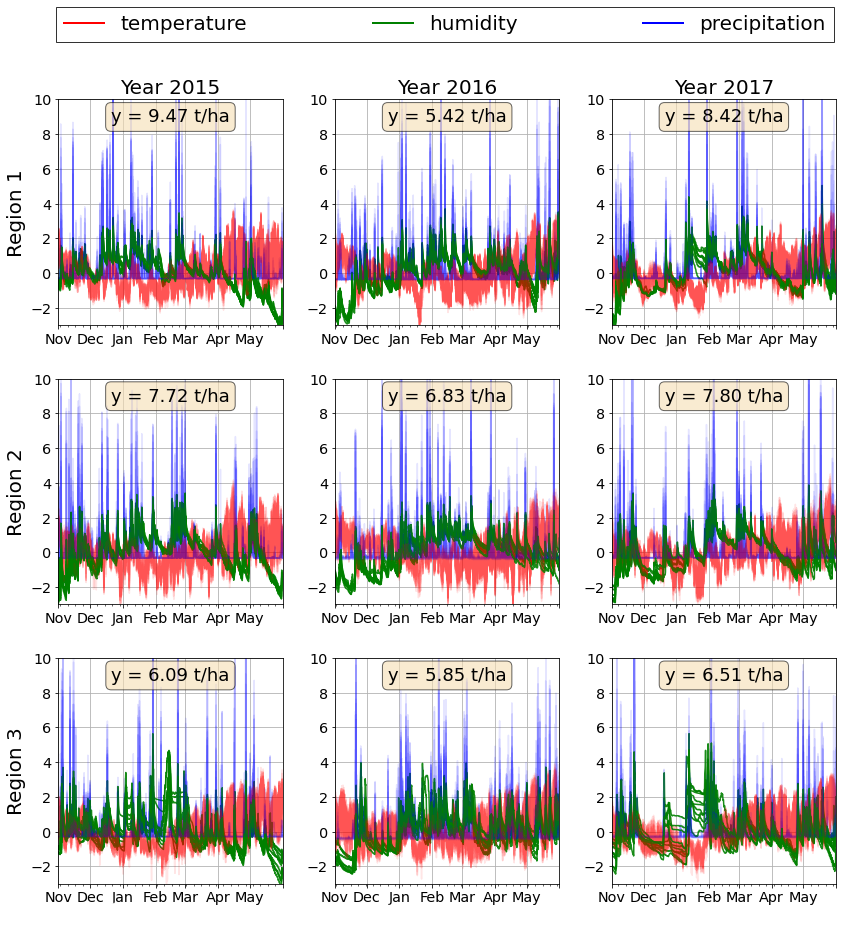}
    \caption{GLDAS/Eurostat dataset. Each plot shows the normalized time-series of temperature, humidity and precipitation, measured over $10$ different locations across a region within a year.}
    \label{fig:crops}
\end{figure}

\end{document}